\documentclass[letterpaper, 10 pt, conference]{ieeeconf}   
\IEEEoverridecommandlockouts     

\usepackage{amsmath,amssymb}   
\usepackage{color}
\usepackage{latexsym} 
\usepackage[normalem]{ulem}   
\usepackage{cite}
\usepackage{graphicx}
\usepackage[english]{babel}    
\usepackage{graphicx}
\usepackage{multirow}
\usepackage{hhline}
\usepackage{siunitx}
\usepackage{pdfpages}
\usepackage{soul}
\usepackage[small]{caption}
\usepackage{url}
\usepackage[labelfont=bf]{caption}

\usepackage[ruled,vlined]{algorithm2e}
\newtheorem{theorem}{Theorem}

\newtheorem{lemma}[theorem]{Lemma}
\newtheorem{definition}{Definition}
\newtheorem{remark}{Remark}
\newtheorem{example}{Example}
\newtheorem{problem}{Problem}

\title{Human-in-the-Loop Mixed-Initiative Control under Temporal Tasks}
\author{Meng Guo, Sofie Andersson and Dimos V. Dimarogonas %
\thanks{The authors are with the ACCESS Linnaeus Center, School of Electrical Engineering, KTH Royal Institute of Technology, SE-100 44, Sweden (Email: {\tt\small mengg, sofa, dimos@kth.se}). This work is supported in part by the H2020 ERC Starting Grant BUCOPHSYS, the EU H2020 Research and Innovation Programme under GA No.731869 (Co4Robots), SSF COIN project and the Knut and Alice Wallenberg Foundation (KAW).}
}

\begin{document}
\maketitle \thispagestyle{empty} \pagestyle{empty}
\begin{abstract}
This paper considers the motion control and task planning problem of mobile robots under complex high-level tasks and human initiatives. The assigned task is specified as Linear Temporal Logic (LTL) formulas that consist of hard and soft constraints. 
The human initiative influences the robot autonomy in two explicit ways: with additive terms in the continuous controller and with contingent task assignments. We propose an online coordination scheme that encapsulates (i) a mixed-initiative continuous controller that ensures all-time safety despite of possible human errors, (ii) a plan adaptation scheme that accommodates new features discovered in the workspace and short-term tasks assigned by the  operator during run time, and (iii) an iterative inverse reinforcement learning (IRL) algorithm that allows the robot to asymptotically learn the human preference on the parameters during the plan synthesis.  The results are demonstrated by both realistic human-in-the-loop simulations and experiments.
\end{abstract}
\section{Introduction}\label{sec:introduction}
Autonomous systems are becoming increasingly prevalent in our daily life, with examples such as self-driving vehicles, package delivery drones and household service robots~\cite{dunbabin2012robots}. Nevertheless, these autonomous systems often perform the intended tasks under the supervision or collaboration with human operators~\cite{fong2003survey}. 
On the high level, the human operator could assign tasks for the robot to execute or monitor the task execution progress during run time. On the low level, the  operator could directly influence or even overtake the control commands of the robot from the on-board autonomous controller, which can be useful to guide the robot through difficult parts of the task~\cite{fong2003survey, goodrich2007human, loizou2007mixed}. On the other hand, the autonomous controller should take into account possibly erroneous inputs from the operator and ensure that safety constraints are never violated. Thus, addressing properly these online interactions between the autonomous system and the  operator during the design process is essential for the safety and efficiency of the overall system. 

In this work, we consider the interactions on both levels. Particularly, on the high level, the operator assigns (i) offline a local task as LTL formulas for hard and soft task constraints, and (ii) online temporary tasks with deadlines.  On the low level, the operator's control inputs is fused directly with the autonomous controller via a mixed-initiative controller. The proposed motion and task planning scheme ensures that the hard task constraints regarding safety are obeyed at all time, while the soft constraints for performance are improved gradually as the robot is guided to explore the workspace and more importantly, learn about the human preference over the synthesized plan.

We rely on LTL as the formal language~\cite{baier2008principles} to describe complex high-level tasks beyond the classic point-to-point navigation.
Many recent papers can be found that combine robot motion planning with model-checking-based task planning, e.g., a single robot under LTL motion tasks~\cite{fainekos2009temporal, bhatia2010sampling, ding2011mdp, belta2017formal}, a multi-robot system under a global task~\cite{ ulusoy2013optimality}, or a multi-robot system under local independent~\cite{guo2015multi} or dependent tasks~\cite{guo2017task,tumova2016multi}.
However, none of the above addresses directly the human initiative, neither in the continuous control nor in the discrete planning. On the other hand, human inputs are considered in~\cite{kress2009temporal} via  GR(1) task formulas that require the robot to be reactive to simple sensory signals from the human.
The high-level robot-human interaction is modeled as a two-player Markov Decision Process (MDP) game in~\cite{feng2016synthesis, fu2015pareto} where they take turns to influence the system evolution. The goal is to design a shared control policy to satisfy a LTL formula and minimize a cost function. 
Another recent work~\cite{jansen17synthesis} addresses the control problem of MDP under LTL formulas where the autonomy strategy is blended into the human strategy in a minimal way that also ensures safety and performance.
However, the direct interaction on the low level is not investigated in the aforementioned work. More importantly, an all-time involvement of the human is required for these frameworks, while we only assume human intervention whenever preferred by the operator itself. 

Furthermore, the notion of mixed-initiative  controller is firstly proposed in~\cite{loizou2007mixed} that combines external human inputs with the traditional navigation controller~\cite{koditschek1990robot}, while ensuring safety of the overall system. The work in~\cite{panagou2013multi, wang2016multi} proposes a systematic way to compose multiple control initiatives using barrier functions.  However, high-level complex temporal tasks are not considered in these work.

The main contribution of this work lies in a novel human-in-the-loop control framework that  allows human interaction on both high level as complex tasks and low level as continuous inputs. We ensure all-time safety during the interaction and accommodation of short-term contingent tasks assigned during run time. Lastly, the proposed IRL algorithm enables the robot to asymptotically learn and adapt to the human operator's preference in the plan synthesis.

The rest of the paper is organized as follows: Section~\ref{sec:prelims}  introduces some preliminaries of LTL. Section~\ref{sec:problem-formulate} formulates the problem. Main algorithmic parts are presented in Section~\ref{sec:blocks}, which are integrated into the complete framework in Section~\ref{sec:complete}.  Numerical and experiment studies are shown in Sections~\ref{sec:case}. We conclude in Section~\ref{sec:future}.

\section{Preliminaries}\label{sec:prelims}

\subsection{Linear Temporal Logic (LTL)}\label{subsec:LTL}
A LTL formula over a set of {atomic propositions} $AP$ that can be evaluated as true or false is defined inductively according to the following syntax~\cite{baier2008principles}:
$\varphi::=\top \;|\; a  \;|\; \varphi_1 \wedge \varphi_2  \;|\; \neg \varphi  \;|\; \bigcirc \varphi  \;|\;  \varphi_1\, \textsf{U}\, \varphi_2,$
where $\top\triangleq \texttt{True}$, $a \in AP$ and $\bigcirc$ (\emph{next}), $\textsf{U}$ (\emph{until}). There are other useful operators like $\square$ (\emph{always}), $\Diamond$ (\emph{eventually}), $\Rightarrow$ (\emph{implication}).
The full semantics and syntax of LTL are omitted here due to limited space, see e.g.,~\cite{baier2008principles}.
{Syntactically co-safe LTL (sc-LTL)} is a subclass of LTL that can be fulfilled by finite satisfying prefix~\cite{kupferman2001model}.

\subsection{B\"uchi Automaton}\label{subsec:buchi}
Given a LTL formula~$\varphi$, there always exists a Nondeterministic B\"uchi Automaton (NBA) $\mathcal{A}_{\varphi}$ that accepts all the languages  that satisfy~$\varphi$~\cite{baier2008principles}.  It is defined as $\mathcal{A}_{\varphi}=(Q, \,2^{AP},\, \delta,\, Q_0,\, \mathcal{F})$, where $Q$ is a finite set of states; $Q_0 \subseteq Q$ is the set of initial states, $2^{AP}$ is the set of input alphabets; $\delta: Q\times 2^{AP} \rightarrow {2^Q}$ is a transition relation and $\mathcal{F}\subseteq Q$ is a set of accepting states. 
There are fast translation algorithms  \cite{gastin2001fast} to obtain~$\mathcal{A}_{\varphi}$.
Moreover, denote by $\chi(q_m,\, q_n)=\{\ell\in 2^{AP}| \, q_n\in \delta (q_m,\,\ell)  \}$ the set of all input alphabets that enable the transition from $q_m$ to $q_n$ in $\delta$.  Then the distance between~$\ell \in 2^{AP}$ and $\chi \subseteq 2^{AP}$($\chi \neq \emptyset$) is defined by $\texttt{Dist}(\ell, \, \chi)=0$ if $\ell \in \chi$ and $\min_{\ell'\in \chi}\; |\{a\in AP\,|\,a\in \ell, a\notin \ell'\}|$, otherwise. Namely, it returns the minimal difference between~$\ell$ and any element in~$\chi$.

\section{Problem Formulation}\label{sec:problem-formulate}

\subsection{Dynamic Workspace and Motion Abstraction}\label{subsec:dynamic-workspace}
The bounded workspace where the robot is deployed is denoted by $\mathcal{W}\subset \mathbb{R}^2$. It consists of $N>0$ regions of interest, denoted by~$\Pi=\{\pi_1,\pi_2,\cdots,\pi_N\}$, where~$\pi_{n}\subset \mathcal{W}$.
Furthermore, there is a set of $M>0$ properties (atomic propositions) associated with $\Pi$, denoted by~$AP=\{a_0,a_1,\cdots,a_M\}$, e.g., ``this is a public area'', ``this is office room one'' and ``this meeting room is in use''.

The robot's motion within the workspace is abstracted as a labeled transition system~$\mathcal{T}\triangleq (\Pi,\,\rightarrow,\, \Pi_0,\,AP,\, L)$, where~$\Pi,\, AP$ are defined above, $\rightarrow \subseteq \Pi \times \Pi$ is the transition relation that~$(\pi_i,\pi_j)\in \rightarrow$ if the robot can move from region~$\pi_i$ to region~$\pi_j$ without crossing other regions in~$\Pi$, $\Pi_0\in \Pi$ is where the robot starts initially, $L: \Pi\rightarrow 2^{AP}$ is the labeling function where~$L(\pi_i)$ returns the set of properties satisfied by~$\pi_i$. Since the workspace is assumed to be only \emph{partially-known} and dynamic, the labeling function and the transition relation may change over time.

\subsection{Mixed-initiative Controller}\label{subsec:robot-control}
For the simplicity of discussion, we assume that the robot satisfies the single-integrator dynamics, i.e., $\dot{x}=u$, where~$x,\,u\in \mathbb{R}^2$ are the robot position and control inputs. For each transition~$(\pi_s,\pi_g)\in \rightarrow$, the robot is controlled by the mixed-initiative navigation controller~\cite{loizou2007mixed} below:
\begin{equation}\label{eq:u-model}
u \triangleq u_r(x,\pi_s,\pi_g) + \kappa(x,\Pi)\, u_h(t)
\end{equation}
where $u_r(x,\pi_s,\pi_g)\in \mathbb{R}^2$ is a given autonomous controller that navigates the robot from region $\pi_s$ to $\pi_g$, while staying within~$\mathcal{W}$ and without crossing other regions in~$\Pi$; the function $\kappa(x,\Pi)\in [0,\; 1]$ is a smooth function to be designed; and $u_h(t)\in \mathbb{R}^2$ is the human input function, which is \emph{uncontrollable} and \emph{unknown} by the robot. 

\begin{remark}\label{remark:model}
The proposed motion and task coordination scheme can be readily extended to robotic platforms with other dynamics and different navigation controllers, such as potential-field-based~\cite{loizou2007mixed} and sampling-based~\cite{bhatia2010sampling,lavalle2006planning}. \hfill $\blacksquare$
\end{remark}

\subsection{Robot Task Assignment}\label{sec:robot-task}
The robot is assigned by the human operator a local task as LTL formulas over~$AP$, which has the following structure:
\begin{equation}\label{eq:ltl-task}
\varphi \triangleq \varphi_{\text{hard}} \wedge \varphi_{\text{soft}} \wedge \varphi_{\text{temp}}
\end{equation}
where $\varphi_{\text{soft}}$ and $\varphi_{\text{hard}}$ are ``soft'' and ``hard'' sub-formulas that are assigned \emph{offline}. Particularly, $\varphi_{\text{hard}}$ includes safety constraints such as collision avoidance: ``avoid all obstacles" or power-supply guarantees: ``visit the charging station infinitely often"; $\varphi_{\text{soft}}$ contains additional requirements for performance such as surveillance: ``surveil all bases infinitely often". Introducing soft and hard constraints is due to the observation that the partially-known workspace might render parts of the task infeasible initially and thus yielding the need for them to be relaxed, while the safety-critical parts should not be relaxed;
Lastly, $\varphi_{\text{temp}}$ contains short-term contingent tasks that are assigned as sc-LTL formulas \emph{online} and {unknown} beforehand. 
The structure in~\eqref{eq:ltl-task} provides an effective way for the operator to handle both standard operational tasks and contingent demands.

\subsection{Control Objective}\label{sec:control-objective}
Given the abstraction model~$\mathcal{T}$ and the task formula~$\varphi$, the control objective is to design function~$\kappa(\cdot)$ and control input~$u$ in~\eqref{eq:u-model} such that:
(I) the hard constraints in~$\varphi_{\text{hard}}$ are always satisfied, given all possible human inputs; 
(II) each time a temporary task~$\varphi_{\text{temp}}$ is assigned, it is satisfied in finite time; and 
(III) the satisfaction of the soft constraints in~$\varphi_{\text{soft}}$ adapts to the human inputs.

\section{Algorithmic Components}\label{sec:blocks}
In this section, we present the four algorithmic components of the overall solution presented in Section~\ref{sec:complete}. Particularly, we start from constructing a parameterized product automaton for the plan synthesis. Then we present a mixed-initiative  controller that guarantees safety and  meaningful inputs from the  operator.  Furthermore, we discuss a plan adaptation algorithms for real-time updates of the workspace model and contingent task assignment. At last, we describe a  IRL algorithm to learn about the human preference.

\subsection{Initial Discrete Plan Synthesis}\label{subsec:init-syn}

Denote by $\mathcal{A}_{\text{hard}}=(Q_1, \,2^{AP},\, \delta_1,\, Q_{1,0},\, \mathcal{F}_1)$ and $\mathcal{A}_{\text{soft}}=(Q_2, \,2^{AP},\, \delta_2,\, Q_{2,0},\, \mathcal{F}_2)$ as the NBAs associated with $\varphi_{\text{hard}}$ and $\varphi_{\text{soft}}$, respectively, where the notations are defined analogously as in Section~\ref{subsec:buchi}.
Now we propose a way to compose~$\mathcal{T}$, $\mathcal{A}_{\text{hard}}$ and $\mathcal{A}_{\text{soft}}$ into a product automaton.

\begin{definition}\label{def:para-prod}
The \emph{parameterized} product automaton $\mathcal{A}_{p}\triangleq (Q_p,\, \delta_p, \, Q_{p,0},\, \mathcal{F}_p)$ is defined as:
$Q_p=\Pi \times Q_1 \times Q_2 \times \{1,2\}$ are the states with $q_p= \langle \pi,\, q_1,\,q_2,c \rangle \in Q_p$,  $\forall \pi\in \Pi$, $\forall q_1 \in Q_1$, $\forall q_2 \in Q_2$ and $\forall c\in \{1,2\}$;
$\delta_p: Q_p \times Q_p \rightarrow (\mathbb{R}_{\geq 0}\cup \{\infty\})^3$ maps each transition to a column vector such that $\delta_p(\langle \pi,\, q_1,\,q_2,c \rangle, \langle \check{\pi},\, \check{q}_1,\,\check{q}_2,\check{c} \rangle) = [\alpha_1,\alpha_2,\alpha_3]^\intercal$, where 
\begin{itemize}
\item $\alpha_1$ is the control cost for the robot to move from~$\pi$ to $\check{\pi}$, where $\alpha_1>0$ if~$(\pi,\,\check{\pi})\in \rightarrow$, otherwise $\alpha_1 \triangleq \infty$;
\item $\alpha_2$ is the indicator for whether a transition violates the hard constraints. It satisfies that $\alpha_2\triangleq 0$ if the following conditions hold: (i)~$L(\pi)\in \chi_1(q_1,\,\check{q}_1)$; (ii) $\chi_2(q_2,\,\check{q}_2)\neq \emptyset$; (iii) $q_1\notin \mathcal{F}_1$ and $\check{c}=c=1$; or $q_2\notin \mathcal{F}_2$ and $\check{c}=c=2$; or $q_1\in \mathcal{F}_1$, $c=1$ and $\check{c}=2$; or $q_2\in \mathcal{F}_2$, $c=2$ and $\check{c}=1$. Otherwise, $\alpha_2\triangleq \infty$.
\item $\alpha_3$ is the measure of how much a transition violates the soft constraints, where~$\alpha_3\triangleq \texttt{Dist}(L(\pi),\,\chi_2(q_2,\check{q}_2))$, where the functions~$\texttt{Dist}(\cdot)$ and $\chi_2(\cdot)$ for $\mathcal{A}_{\text{soft}}$ are defined in Section~\ref{subsec:buchi}.
\end{itemize}
and $Q_{p,0}= \Pi_{0}\times Q_{1,0}\times Q_{2,0}\times \{1\}$, $\mathcal{F}_p=\Pi \times \mathcal{F}_1 \times Q_2 \times \{1\}$ are the sets of initial and accepting states, respectively. \hfill $\blacksquare$
\end{definition}

An accepting run of~$\mathcal{A}_p$ is an infinite run that starts from any initial state and intersects with the accepting states infinitely often. Note that the component $c$ above to ensure that an accepting run intersects with the accepting states of both $\mathcal{A}_{\text{hard}}$ and $\mathcal{A}_{\text{soft}}$ infinitely often. More details can be found in Chapter 4 of~\cite{baier2008principles}. Furthermore, since the  workspace~$\mathcal{T}$ is partially-known, we denote by $\mathcal{T}^t$ the workspace model at time $t\geq 0$, and the associated product automaton by~$\mathcal{A}_p^t$.

To simplify the notation, given a finite run~$R=q_p^0q_p^1\cdots q_p^S$ of $\mathcal{A}_p$, where $q_p^s\in Q_p$, $\forall s=0,1,\cdots,S$,  we denote by
$\boldsymbol{\delta}(R)=\sum_{s=0}^{S-1}\delta_{p}(q_p^{s},q_p^{s+1}),
$
where~$\boldsymbol{\delta}(R)\in \mathbb{R}^3$ is the accumulated cost vector~$\delta_p$ along~$R$. Similar definitions hold for~$\boldsymbol{\alpha}_{k}(R)\in \mathbb{R}$ as the accumulated $\alpha_{k}$ cost along~$R$, $\forall k=1,2,3$.
 We consider an accepting run of~$\mathcal{A}_p$ with the prefix-suffix structure:~$R_p\triangleq q_p^1 q_p^2\cdots q_p^S\big{(}q_p^{S+1} q_p^{S+2}\cdots q_p^{S+F}\big{)}^\omega$, where~$q_p^j\in Q_p$, $\forall j=1,2,\cdots,S+F$, where~$S,F>0$. The plan prefix $R_p^{\text{pre}}\triangleq q_p^1 q_p^2\cdots q_p^S$ is executed only once while the plan suffix $R_p^{\text{suf}} \triangleq q_p^{S+1} q_p^{S+2}\cdots q_p^{S+F}$ is repeated infinitely often. Then the total cost of~${R}_p$ is defined as:
\begin{equation}\label{eq:plan-cost}
\texttt{C}_{\beta}(R_p) \triangleq [1,\,\; \gamma] \otimes \begin{bmatrix} 1 \\ 1 \\ \;\beta \end{bmatrix}^\intercal \cdot
\begin{bmatrix} \boldsymbol{\delta}(R_p^{\text{pre}})\\
 \boldsymbol{\delta}(R_p^{\text{suf}})
\end{bmatrix},
\end{equation}
where $\texttt{C}_{\beta}(R_p)\geq 0$;~$\otimes$ is the Kronector product; $\gamma\geq 0$ is a weighting parameter between the cost of the plan prefix and suffix; $\beta\geq 0$ is a weighting parameter between total control cost of the plan and the satisfaction of the soft task~$\varphi_{\text{soft}}$. Note that~$\gamma$ is normally constant~\cite{guo2015multi} (set to $1$ in this work), while $\beta$ can change according to the robot's internal model or the operator's \emph{preference}. For instance, as the robot has more accurate workspace model, $\beta$ can be increased to penalize the violation of~$\varphi_{\text{soft}}$ such that~$R_p$ satisfies~$\varphi_{\text{soft}}$ more. Or the operator prefers that $\beta$ is decreased so that $R_p$ satisfies~$\varphi_{\text{soft}}$ less and the robot reserves more power.

Given the initial values of~$\gamma$ and~$\beta$, an initial accepting run of~$\mathcal{A}_p$, denoted by~$R_p^0$, can be found that {minimizes} the total cost in~\eqref{eq:plan-cost}. The algorithms are based on the nested Dijkstra's search, which are omitted here and details can be found in~\cite{guo2015multi}.
As a result, the robot's initial plan, denoted by~$\tau_r^0$, can be derived by projecting~$R_p^0$ onto~$\Pi$, as a sequence of regions that the robot should reach:
$\tau_r^0=\pi^1\pi^2\cdots \pi^S\big{(}\pi^{S+1}\pi^{S+2}\cdots \pi^{S+F}\big{)}^{\omega}$,
where~$\pi^j$ is the projection of $q_p^j$ onto $\Pi$, $\forall j=1,2,\cdots,S+F$.

\subsection{Mixed-initiative Controller Design}\label{subsec:design-control}
After the system starts, the robot executes the initial plan~$\tau_r^0$ by reaching the sequence of regions defined by it. 
However, as described in Section~\ref{subsec:robot-control}, the robot controller is also influenced by the human input. In the following, we show how to construct function~$\kappa(\cdot)$ in~\eqref{eq:u-model} such that the hard task~$\varphi_{\text{hard}}$ is respected at all times for all human inputs. 

First, we need to find the set of product states~$\mathcal{O}_t\subset Q_p$ in~$\mathcal{A}_p^t$ at time~$t\geq 0$, such that once the robot belongs to any state in~$\mathcal{O}_t$ it means that~$\varphi_{\text{hard}}$ can not be satisfied any more.

\begin{lemma}\label{lem:unsafe}
Assume that the robot belongs to state~$q_{p}\in Q_p$ at time $t>0$. Then the hard task~$\varphi_{\text{hard}}$ can not be satisfied in the future, if $\mathcal{A}_p^t$ remains unchanged and the minimal cost of all paths from~$q_p$ to any accepting state in~$\mathcal{F}_p$ is~$\infty$.
\end{lemma}
\begin{proof}
Omitted as it is a simple inference of~\eqref{eq:plan-cost}.
\end{proof}

Thus denote by~$Q_t\subset Q_p$ the set of \emph{reachable} states by the robot at time $t>0$. For each~$q_p\in Q_t$, we perform a Dijkstra search to compute the shortest distance from $q_p$ to all accepting states in~$\mathcal{F}_p$. Lastly, $\mathcal{O}_t$ is given as the subset of $Q_t$ that have an infinite cost to \emph{all} accepting states, i.e.,
\begin{equation}\label{eq:ot}
\mathcal{O}_t = \{q_p\in Q_t\,|\,\texttt{C}_{\beta}(\overline{R}_{q_p,\,q_F})=\infty, \forall q_F \in \mathcal{F}_p\},
\end{equation}
where~$\overline{R}_{q_p,\,q_F}$ is the shortest path from~$q_p$ to~$q_F$. 

Given~$\mathcal{O}_t$ above, we now design the function $\kappa(x,\Pi)$ in~\eqref{eq:u-model} such that $\mathcal{O}_t$ can be avoided. Consider the function:
\begin{equation}\label{eq:kappa}
\kappa(x,\Pi) = \kappa(x,\mathcal{O}_t) \triangleq  \frac{\rho(d_t-d_s)}{\rho(d_t-d_s)+\rho(\varepsilon+d_s-d_t)}
\end{equation}
where $d_t\triangleq \textbf{min}_{\langle \pi,q_1,q_2,c\rangle\in \mathcal{O}_t}\; \|x-\pi\|$ is the minimum distance between the robot and any region within~$\mathcal{O}_t$; $\rho(s)\triangleq e^{-1/s}$ for~$s>0$ and $\rho(s)\triangleq 0$ for~$s\leq 0$, and $d_s,\, \varepsilon >0$ are design parameters as the safety distance and a small buffer. Thus the  mixed-initiative controller is given by
\begin{equation}\label{eq:mixed-init}
u \triangleq u_r(x,\pi_s,\pi_g) + \kappa(x,\mathcal{O}_t)\, u_h(t).
\end{equation}
As discussed in~\cite{loizou2007mixed}, the function~$\kappa(\cdot)$ above is $0$ on the boundary of undesired regions in~$\mathcal{O}_t$, and close to $1$ when the robot is away from~$\mathcal{O}_t$ to allow for meaningful inputs from the  operator. This degree of closeness is tunable via changing~$d_t$ and~$d_s$. However, the original definition in~\cite{loizou2007mixed} only considers static obstacles, instead of the general set~$\mathcal{O}_t$.

\begin{lemma}\label{lemma:iss}
Assume that the navigation control $u_r$ is perpendicular to the boundary of regions in $\mathcal{O}_t$ and points inwards. Then the robot can avoid~$\mathcal{O}_t$ under the mixed-initiative controller by~\eqref{eq:mixed-init} for all human input~$u_h$.
\end{lemma}
\begin{proof}
The proof follows from Proposition 3 of~\cite{loizou2007mixed}. Namely, for~$x \in \partial \mathcal{O}_t$ on the boundary of~$\mathcal{O}_t$, 
$$\dot{x}^\intercal u_r = \|u_r(x)\|^2 +\kappa(x)\, u_h(t)^\intercal u_r(x)>0$$
since~$\kappa(x)=0$ for~$x\in \partial \mathcal{O}_t$. Thus the workspace excluding all regions in~$\mathcal{O}_t$ is positive invariant under controller~\eqref{eq:mixed-init}. In other words, 
if the navigation control avoids $\mathcal{O}_t$, the same property is ensured by~\eqref{eq:mixed-init} for all human inputs.
\end{proof}

\subsection{Discrete Plan Adaptation}\label{subsec:plan-adapt}
In this section, we describe how the  discrete plan~$\tau_r^t$ at $t\geq 0$ can be updated to (i) accommodate changes in the partially-known workspace model, and (ii) fulfill contingent tasks that are assigned by the  operator during run time.

\subsubsection{Updated Workspace Model}\label{subsub:update-ws}
The robot can explore new features of the workspace while executing the discrete plan $\tau_r^t$ or being guided by the human operator. Thus the motion model~$\mathcal{T}^t$ can be updated as follows: (i) the transition relation~$\rightarrow$ is modified based on the status feedback from the navigation controller, i.e., whether it is feasible to navigate from region~$\pi_i$ to $\pi_j$ without human inputs; (ii) the labeling function $L(\pi)$ is changed based on the feedback from the robot's sensing module, i.e., the properties that region~$\pi$ satisfies. For example, ``the corridor connecting two rooms is blocked'' or ``the object of interest is no longer in one room''.
Given the updated $\mathcal{T}^t$, the mapping function $\delta_p$ of the product automaton~$\mathcal{A}^t_p$ is re-evaluated. 
Consequently, the current plan~$\tau_r^t$ might not be optimal anymore regarding the cost in~\eqref{eq:plan-cost}. Thus we consider the following problem.
\begin{problem}\label{prob:update-ws}
Update~$\tau_r^t$ such that it has the minimum cost in~\eqref{eq:plan-cost} given the updated model~$\mathcal{T}^t$. \hfill $\blacksquare$
\end{problem}

Given the set of reachable states~$Q_t\subset Q_p$ at time $t>0$, for each state~$q_p \in Q_t$, we perform a nested Dijkstra's search~\cite{guo2015multi} to find the accepting run  that starts from~$q_p$ and has the prefix-suffix structure with the minimum cost defined in~\eqref{eq:plan-cost}. Denote by~$R^t_{+}(q_p)$ this optimal run for~$q_p\in Q_t$. Moreover, let $\mathbf{R}^t_{+}\triangleq \{R^t_{+}(q_p),\, \forall q_p \in Q_t\}$
collect all such runs. Then we find among $\mathbf{R}^t_{+}$ the accepting run with the minimum cost, which becomes the updated run~$R_p^{t+}$:
\begin{equation}\label{eq:update-run}
R_p^{t+} \triangleq \textbf{argmin}_{R_p\in \mathbf{R}^{t}_+} \texttt{C}_{\beta}(R_p).
\end{equation}
Thus the updated plan~$\tau_r^t$ is given by the projection of~$R_p^{t+}$ onto~$\Pi$. Note that the above procedure is performed whenever the motion model~$\mathcal{T}^t$ is updated.

\subsubsection{Contingent Task Fulfillment}\label{subsubsec:cont-task}
As defined in~\eqref{eq:ltl-task}, the operator can assign contingent and short-term tasks~$\varphi_{\text{temp}}$ to the robot during run time. Particularly, we consider the following ``\emph{pick-up and deliver}'' task with deadlines, i.e., 
\begin{equation}\label{eq:temp-task}
(\varphi_{\textup{temp}}^t,\, T_{sg})\triangleq (\Diamond (\pi_s \wedge \Diamond \pi_{g}),\,T_{sg}),
\end{equation}
where~$\varphi_{\textup{temp}}^t\triangleq \Diamond (\pi_s \wedge \Diamond \pi_{g})$ is the temporary task assigned at time $t>0$, meaning that the robot needs to pick up some objects at region $\pi_s$ and deliver them to $\pi_g$ (note that action propositions are omitted here, see~\cite{guo2017task}), where~$\pi_s,\pi_g\in \Pi$; $T_{sg}>0$ is the \emph{preferred} deadline that task~$\varphi_{\textup{temp}}^t$ is accomplished.
It can be verified that $\varphi_{\textup{temp}}^t$ are sc-LTL formulas and can be fulfilled in finite time. Assume that $\varphi_{\textup{temp}}^t$ is satisfied at time $t'>t$ then the delay is defined as $\overline{t}_{sg}\triangleq t'-T_{sg}$. We consider to following problem to incorporate $\varphi_{\textup{temp}}^t$.

\begin{problem}\label{prob:fulfill-temp}
Update~$R_p^t$ such that $\varphi_{\textup{temp}}^t$ can be fulfilled \emph{without} delay (if possible) and with \emph{minimum} extra cost, while respecting the hard constraints~$\varphi_{\textup{hard}}$. \hfill $\blacksquare$
\end{problem}

Assume the remaining optimal run at time $t>0$ is given by~$R_p^t=q^{k_0}_p q_p^{k_0+1}\cdots(q_p^S\cdots q_p^{S+F})^\omega$ and $q^{k_s}_p,q^{k_g}_p\in Q_p$ are the $k_s$-th, $k_g$-th state of~$R_p^t$, where $k_s\geq k_g\geq k_0$. Since~$R_p^t$ is optimal for the current~$\mathcal{T}^t$, we search for the index~$k_s$ where the robot can deviate from~$R_p^t$ to reach region~$\pi_s$ and back, and another index~$k_g$ where the robot can deviate from~$R_p^t$ to reach~$\pi_g$ and back. Denote by~$R_p^{t+}$ the updated run after incorporating~$\pi_s,\pi_g$. 
In this way,~$\varphi_{\textup{temp}}^t$ is satisfied when~$\pi_{g}$ is reached after $\pi_s$ is reached, where~$t'=\sum_{j=k_0}^{k_g}\alpha_1(q_p^j,q_p^{j+1})$ is the total time. Moreover, the total cost of~$R^t_p$ in~\eqref{eq:plan-cost} is changed by~$\overline{\texttt{C}}_{\beta} (R^t_p) \triangleq \texttt{C}_{\beta}(R^{t+}_p)-\texttt{C}_{\beta}(R^t_p)$. 
Thus we formulate a 2-stage optimization below: first, we solve 
\begin{subequations}\label{eq:delay}
\begin{align}
\overline{d}_{sg} &= \textbf{min}_{\{k_g>k_s\geq 0\}}\;  \{\overline{\texttt{C}}_{\beta}(R^t_p)\},\quad \textbf{s.t.}\quad \overline{t}_{sg} \leq 0,  \label{eq:delay-1}\\
\intertext{in order to find out whether it is possible to avoid delay while satisfying $\varphi_{\textup{temp}}$. If no solution is found, we solve the relaxed optimization that allows the deadline to be missed:}
\overline{d}_{sg} &= \textbf{min}_{\{k_g>k_s\geq 0\}}\; \{\overline{t}'_{sg} + \overline{\texttt{C}}_{\beta}(R^t_p)\},\label{eq:delay-2}
\end{align}
\end{subequations}
where $\overline{d}_{sg}\geq 0$; $\overline{t}'_{sg}=0$ if $\overline{t}_{sg}\leq 0$ and $\overline{t}'_{sg} = \overline{t}_{sg}$, otherwise. 
Note that $\overline{\texttt{C}}_{\beta}(R^t_p)$  is $\infty$ if~$\varphi_{\textup{hard}}$ is violated by~$R_p^{t+}$.

Since the suffix of~$R_p^t$ is repeated infinitely often, the choice of indices $k_s,k_g$ for \eqref{eq:delay} is finite. Thus \eqref{eq:delay} can be solved as follows: starting from $k_s=k_0$, we iterate through $k_g\in \{k_0+1,\cdots,\,S+F\}$ and compute  the corresponding $\overline{t}_{sg}$ and $\overline{d}_{sg}$ for both cases in~\eqref{eq:delay}. Then we increase $k_s$ incrementally by~$k_s=k_0+1$, iterate through $k_g\in [k_0+2,\cdots,\, S+F]$ and compute $\overline{t}_{sg}$ and $\overline{d}_{sg}$ for both cases. This procedure repeats itself \emph{until} $k_s=S+F-1$. Then, we find among these candidates if there is a pair $k^\star_s,k^\star_g$ that solves~\eqref{eq:delay-1}. If so, they are the optimal choice of $k_s,k_g$. 
Otherwise, we search for the optimal solution to~\eqref{eq:delay-2}, of which the solution always exists as it is unconstrained.
At last, $R_p^{t+}$ is derived by inserting the product states associated with $\pi_s$, $\pi_g$ at indices $k^\star_s$, $k^\star_g$ of $R_p^t$, respectively.

\subsection{Human Preference Learning}\label{subsec:pref-learn}

As discussed in Section~\ref{subsec:design-control}, the mixed-initiative controller~\eqref{eq:mixed-init} allows the  operator to interfere the robot's trajectory such that it deviates from its discrete plan $\tau_r^t$, while always obeying~$\varphi_{\textup{hard}}$. 
This is beneficial as the robot could be guided to (i) explore \emph{unknown} features to update its workspace model, as described in Section~\ref{subsub:update-ws}; and (ii) follow the trajectory that is \emph{preferred} by the operator.

Particularly, as discussed in Section~\ref{subsec:init-syn}, the initial run~$R_p^0$ is a balanced plan between reducing the control cost and improving the satisfaction of~$\varphi_{\textup{soft}}$, where the weighting parameter is $\beta$ in~\eqref{eq:plan-cost}. Clearly, different choices of~$\beta$ may result in different~$R^0_p$. The initial plan~$R_p^0$ is synthesized under the initial value $\beta_0\geq 0$, which however might \emph{not} be what the operator prefers. In the following, we present how the robot could learn about the {preferred}~$\beta$ from the operator's inputs during run time.

Consider that at time~$t\geq 0$, the robot's past trajectory is given by~$\zeta|_0^{t}\triangleq \pi_0\pi_1\cdots \pi_{k_t}$. Assume now that during time $[t,\,t']$, where $t'>t>0$, via the mixed-initiative controller in~\eqref{eq:mixed-init}, the operator guides the robot to reach a sequence of regions that s/he prefers, which is defined by: 
\begin{equation}\label{eq:human-traj}
\zeta_h|_t^{t'}\triangleq \pi'_1\pi'_2\cdots \pi'_H
\end{equation}
where~$\pi'_h\in \Pi$, $\forall h=1,2\cdots H$ and~$H\geq 1$ is the length of~$\zeta_h$ that can vary each time the operator acts. Afterwards, the robot continues executing its current plan~$\tau_r^t$.
Thus, the actual robot trajectory until time $t'$ is given by $\zeta_h|_0^{t'}\triangleq \zeta|_0^{t}\,\zeta_h|_t^{t'}$, which is the concatenation of $\zeta|_0^{t}$ and $\zeta_h|_t^{t'}$.

\begin{problem}\label{prob:IRL}
Given the actual robot trajectory~$\zeta_h|_0^{t'}$, design an algorithm to estimate the preferred value of $\beta$ as~$\beta_h^\star$ such that~$\zeta_h|_0^{t'}$ corresponds to the optimal plan under $\beta_h^\star$.   \hfill $\blacksquare$
\end{problem}

The above problem is closely related to the inverse reinforcement learning (IRL) problem~\cite{ng2000algorithms, ratliff2006maximum}, where the robot learns about the cost functions of the system model based on demonstration of the preferred plans. On the other hand, in reinforcement learning~\cite{sutton1998reinforcement,Bertsekas1996} problem, the robot learns the optimal plan given these functions.

As mentioned in~\cite{ng2000algorithms}, most problems of IRL are ill-posed. In our case, it means that there are more than one $\beta_h^\star$ that render $\zeta_h|_0^{t'}$ to be the optimal plan under $\beta_h^\star$. In order to improve \emph{generalization} such that the robot could infer the human preference based on the human's past inputs (instead of simply repeating them), our solution is based on the maximum margin planning algorithm from~\cite{ratliff2006maximum}. The general idea is to iteratively update $\beta$ via a sub-gradient descent, where the gradient is computed based on the difference in cost between $\zeta_h|_0^{t'}$ and the optimal plan under the current $\beta$.

First, we compute the set of all finite runs within~$\mathcal{A}_p^{t'}$,  denoted by $\mathbf{R}_h^{t'}$, that are associated with $\zeta_h|_0^{t'}$.  It can be derived iteratively via a breadth-first graph search~\cite{baier2008principles}.
Among~$\mathbf{R}_h^{t'}$, we find the one with the minimal cost over $\alpha_3$, i.e., 
\begin{equation}\label{eq:minimal}
R^\star_h \triangleq \textbf{argmin}_{R\in \mathbf{R}_h^{t'}}\; \boldsymbol{\alpha}_3 (R).
\end{equation}
Let $R^\star_h\triangleq q_1q_2\cdots q_H$, where $q_h\in Q_p$, $\forall h=1,2,\cdots,H$. 
Denote by $\beta_k$ the value of $\beta$ at the $k$-th iteration, for $k\geq 0$. Note that $\beta_0\triangleq \beta_t$, where $\beta_t$ is the value of $\beta$ at time $t>0$.
For the $k$-th iteration, we find the optimal run from~$q_1$ to $q_H$ under $\beta_k$ with certain margins, i.e., 
\begin{equation}\label{eq:beta-optimal}
\hat{R}_{\beta_k}^\star \triangleq \textbf{argmin}_{R\in \mathbf{R}_{q_1q_H}}\; \Big{(}\texttt{C}_{\beta_k}(R) - M(R,R^\star_h)\Big{)}
\end{equation}
where $\mathbf{R}_{q_1q_H}$ is the set of \emph{all} runs from $q_1$ to $q_H$ in $\mathcal{A}_p^{t'}$; and $M:Q^H\times Q^H\rightarrow \mathbb{N}$ is the margin function~\cite{ratliff2006maximum}:
\begin{equation}\label{eq:margin}
M(R,\,R^\star_h) = |\{(q_s,q_t)\in R\,|\, (q_s,q_t)\notin R_h^\star\}|,
\end{equation}
which returns the number of edges within $R$ that however do not belong to $R_h^\star$. The margin function decreases the total cost $\texttt{C}_{\beta_k}(R)$ by the difference between $R$ and $R_h^\star$. It can improve generalization and help address the ill-posed nature of Problem~\ref{prob:IRL}.
To solve~\eqref{eq:beta-optimal}, we first  modify~$\mathcal{A}_p^{t'}$ by reducing the $\alpha_1$ cost of each edge $(q_s,q_t)\in R_h^\star$ by one. Then a Dijkstra shortest path search can be performed over the modified $\mathcal{A}_p$ to find the shortest run from $q_1$ to $q_H$ that minimizes the cost with margins in~\eqref{eq:beta-optimal}.
Given~$\hat{R}_{\beta_k}^\star$, we can compute the sub-gradient~\cite{shor2012minimization} that $\beta_k$  should follow:
\begin{equation}\label{eq:sub-gradient}
\nabla \beta_k = \lambda\cdot \beta_k + \big{(}\boldsymbol{\alpha}_3(R^\star_h) - \boldsymbol{\alpha}_3(\hat{R}_{\beta_k}^\star)\big{)},
\end{equation}
where~$\nabla\beta_k \in \mathbb{R}$ and $\lambda>0$ is a design parameter. Thus, at this iteration the value of~$\beta_k$ is updated by
\begin{equation}\label{eq:update-beta}
\beta_{k+1}=\beta_{k} - \theta_{k} \cdot \nabla \beta_{k},
\end{equation}
where~$\theta_k>0$ is the step size or learning rate~\cite{sutton1998reinforcement}. Given the updated~$\beta_{k+1}$, the same process in~\eqref{eq:minimal}-\eqref{eq:update-beta} is repeated until the difference $|\beta_{k+1}-\beta_k|$ is less than a predefined threshold~$\varepsilon>0$. At last, the value of $\beta_t$ is updated to $\beta_{k+1}$.
The discussion above is summarized in Alg.~\ref{alg:learn-beta}. Each time the human operator guides the robot to reach a new sequence of regions, the estimation of the value of~$\beta_t$ is updated by running Alg.~\ref{alg:learn-beta}. In the following, we show that Alg.~\ref{alg:learn-beta} ensures the convergence of $\{\beta_k\}$.

\begin{algorithm}[t]
\caption{On-line IRL algorithm for $\beta$.} \label{alg:learn-beta}
\LinesNumbered
\KwIn{$\mathcal{A}^t_p$, $\zeta_h|_0^{t'}$, $\beta_t$, $\varepsilon$}
Initialize $\beta_k=\beta_t$ for iteration $k=0$\;
\While(\tcp*[f]{Iteration $k$}){$|\beta_{k+1}-\beta_k|>\varepsilon$}
{Compute $R_h^\star$ in~\eqref{eq:minimal} given~$\zeta_h|_0^{t'}$\;
Find $\hat{R}_{\beta_k}^{\star}$ in~\eqref{eq:beta-optimal} given $\beta_k$ and $R_h^\star$\;
Compute $\nabla \beta_k$ by~\eqref{eq:sub-gradient} and update $\beta_k$ by~\eqref{eq:update-beta}\;}
\Return $\beta_t^+=\beta_{k+1}$
\end{algorithm}

\begin{lemma}\label{lemma:whole}
The sequence $\{\beta_k\}$ in Alg.~\ref{alg:learn-beta} converges to a fixed $\beta_l^\star\geq 0$ and the optimal plan under $\beta_l^\star$ is $\zeta_h|_0^{t'}$.
\end{lemma}
\begin{proof}
Firstly, the optimal run $R_h^\star$ associated with $\zeta_h|_0^{t'}$ under $\beta_h^\star$ minimizes the balanced cost~$\texttt{C}_\beta$ from~\eqref{eq:plan-cost}, i.e., 
\begin{equation}\label{eq:opt-1}
\texttt{C}_{\beta}(R_h^\star) \leq \texttt{C}_{\beta}(R), \quad \forall R\in \mathbf{R}_{q_1q_H},
\end{equation}
where $\mathbf{R}_{q_1q_H}$ is defined in~\eqref{eq:beta-optimal}. Solving~\eqref{eq:opt-1} directly can be computationally expensive due to the large set~$\mathbf{R}_{q_1q_H}$. We introduce a slack variable~$\xi\in \mathbb{R}$ to relax the constraints:
\begin{equation}\label{eq:opt-2}
\begin{split}
&\textbf{min}_{\beta\geq 0} \quad \frac{\lambda}{2}\beta^2 + \xi \\
& \textbf{s.t.} \quad \texttt{C}_{\beta}(R^\star_h) - \xi \leq \min_{R\in \mathbf{R}_{q_1q_H}} \Big{(}\texttt{C}_{\beta}(R) - M(R,R^\star_h)\Big{)}, \\
\end{split}
\end{equation}
where~$\lambda>0$ is the same as in~\eqref{eq:sub-gradient} and the margin function $M(\cdot)$ is from~\eqref{eq:margin}.
Thus, by enforcing the slack variables to be tight, $\beta$ also {minimizes} the combined cost function:
\begin{equation}\label{eq:combined}
\frac{\lambda}{2}\beta^2 + \texttt{C}_{\beta}(R^\star_h) - \min_{R\in \mathbf{R}_{q_1q_H}} \Big{(}\texttt{C}_{\beta}(R) - M(R,R^\star_h)\Big{)},
\end{equation}
which is convex but non-differentiable. Instead, we compute the sub-gradient~\cite{shor2012minimization} of~\eqref{eq:combined}:
$\nabla \beta = \lambda \beta + \big{(}\boldsymbol{\alpha}_3(R^\star_h) - \boldsymbol{\alpha}_3(\hat{R}_{\beta}^\star)\big{)}$ and $\hat{R}_{\beta}^\star=\textbf{argmin}_{R\in \mathbf{R}_{q_1q_H}} {(}\texttt{C}_{\beta}(R) - M(R,R^\star_h) {)}$, which is equivalent to~\eqref{eq:sub-gradient}.

Lastly, by the strong convexity of~\eqref{eq:combined} and Theorem~1 of~\cite{ratliff2006maximum}, the estimation $\beta_t$ approaches the optimal $\beta_l^\star$ with linear convergence rate under constant stepsize $\theta_k=\theta$, i.e., $|\beta_k-\beta_l^\star|^2\leq (1- \theta\lambda)^{k+1}|\beta_0-\beta_l^\star|^2+\frac{\theta |\nabla \beta|_{\max}}{\lambda}$.
A detailed analysis on this deviation can be found in~\cite{ratliff2006maximum,shor2012minimization}.
\end{proof}

\begin{remark}
It is worth noting that the convergent value $\beta_l^\star$ might be \emph{different} from the preferred~$\beta^\star_h$, while they both satisfy~\eqref{eq:opt-1} with the same optimal run $R_h^\star$. However, the margin function in~\eqref{eq:opt-2} ensures that $\beta_l^\star$ is \emph{better} than or at least equivalent to $\beta_h^\star$ in terms of the similarity between $R_h^\star$ and the run with the second minimum cost by~\eqref{eq:plan-cost}.
\end{remark}

\begin{table}[t]
\begin{center}
\scalebox{1.1}{
    \begin{tabular}{| c| c | c | c | c|}
    \hline
Method &   $|\mathcal{A}_p|$ &  $\beta_l^\star$   & NO. of Dijkstra  &   Time[\si{\second}] \\ \hline \hline
   Alg.1 & 25 & 13.4 & 8 & 3.8\\ \hline
  M1 & 25  & 10.0 & 200 & 124.4 \\ \hline
   M2 & 25 & 11.7 & 350 & 337.2\\ \hline \hline
   Alg.1 & 100 & 16.5 & 12 & 150.8\\ \hline
   M1 & 100  & 14.2 & 200 & 2203.5 \\ \hline
   M2 & 100 & -- & 800+ & 3000+\\ \hline 
    \end{tabular}}
\caption{Comparison of computational complexity and performance of Alg.~\ref{alg:learn-beta} and two alternative methods in Example~\ref{example:compare}.}
\label{table:beta-statistics}
\end{center}
\end{table}

Now we show the computational efficiency of Alg.~\ref{alg:learn-beta} compared with two straight-forward solutions: (M1) choose the optimal $\beta$ among a set of guessed values of $\beta$, denoted by $S_\beta$; (M2) solve~\eqref{eq:opt-1} directly by enumerating all runs in $\mathbf{R}_{q_1q_H}$. The first method's accuracy relies on  $S_\beta$ being large, which however results in high computational cost. Similarly, the second method relies on evaluating \emph{every} run in~$\mathbf{R}_{q_1q_H}$, the size of which is combinatorial to the size of $\mathcal{A}^t_p$. The following example shows some numerical comparison.

\begin{example}\label{example:compare}
Assume that $\beta_h^\star=15$ and initially $\beta_0=0$. We use three methods: Alg.~\ref{alg:learn-beta}, M1 and M2 above to estimate $\beta_h^\star$. As shown in Table~\ref{table:beta-statistics}, we compare the final convergence $\beta_l^\star$ and the computation time under varying sizes of $\mathcal{A}_p$. It can be seen that the computation time for Alg.~\ref{alg:learn-beta} is significantly less than M1 and M2, where for the second case M2 fails to converge within $50\si{\minute}$.\hfill $\blacksquare$
\end{example}

\section{The Integrated System}\label{sec:complete}
In this section, we describe the the real-time execution of the integrated system given the components in Section~\ref{sec:blocks}. Then we discuss the computational complexity.

\subsection{Human-in-the-loop Motion and Task Planning}\label{subsec:summary}
The complete algorithm is shown in Alg.~\ref{alg:complete}. Before the system starts, given the initial model $\mathcal{T}^0$ and the task formulas in~\eqref{eq:ltl-task}, the initial plan~$\tau_r^0$ is synthesized by the algorithm from Section~\ref{subsec:init-syn} under the initial~$\beta_0$. From~$t=0$, the robot executes~$\tau_r^0$ by following the sequence of goal regions, see Lines 1-3.
Meanwhile, the operator can directly modify the control input~$u(\cdot)$ via~\eqref{eq:mixed-init} to change the robot's trajectory. Thus, the robot can explore regions that are not in its initial plan and update the model $\mathcal{T}^t$ as described in Section~\ref{subsub:update-ws}.
As a result, the plan~$\tau_r^t$ is updated by~\eqref{eq:update-run} accordingly, see Lines 4-5. 
Moreover, as described in Section~\ref{subsubsec:cont-task}, the operator can assign temporary tasks with deadlines as in~\eqref{eq:temp-task}, for which~$\tau_r^t$ is modified by solving~\eqref{eq:delay}, see Lines 6-7.
Last but not least, each time the operator guides the robot to follow a new trajectory, the parameter $\beta$ is updated via Alg.~\ref{alg:learn-beta} to estimate the human preference. Then, its current plan~$\tau_r^t$ is updated using the updated~$\beta$, see Lines 8-12. The above procedure repeats until the system is terminated. 

\begin{theorem}\label{theorem:correctness}
Alg.~\ref{alg:complete} above fulfills the three control objectives of Section~\ref{sec:control-objective}, i.e., (I)~$\varphi_{\textup{hard}}$ is satisfied for all time; (II) each~$\varphi_{\textup{temp}}$ is satisfied in finite time; and (III) the satisfaction of $\varphi_{\textup{soft}}$ adapts to the human inputs. 
\end{theorem}
\begin{proof}
(I) Firstly, both the initial synthesis algorithm in Section~\ref{subsec:init-syn} and the plan adaptation algorithm in Section~\ref{subsub:update-ws} ensure $\varphi_{\textup{hard}}$ is satisfied by minimizing the total cost in~\eqref{eq:plan-cost}. Then Lemma~\ref{lemma:iss} ensures that~$\varphi_{\textup{hard}}$ is respected for all possible inputs from the human operator. (II) The combined cost~\eqref{eq:delay} ensures that~$\varphi_{\textup{temp}}$ is satisfied within finite time. (III) Convergence of the learning Alg.~\ref{alg:learn-beta} is shown in Lemma~\eqref{lemma:whole}. Thus, the updated plan~$\tau_r^t$ under the learned value of $\beta$ adapts to the plan preferred by the operator.
\end{proof}

\begin{algorithm}[t]
\caption{Mixed-initiative Motion and Task Planning} \label{alg:complete}
\LinesNumbered
\KwIn{$\mathcal{T}^t$, $\varphi_{\text{hard}}$, $\varphi_{\text{soft}}$, $\beta_0$, $u_h(t)$, $(\varphi^t_{\text{temp}},T_{sg})$}
Compute~$\mathcal{A}_p^0$ and construct initial plan~$\tau_r^0$ under $\beta_0$\;
\ForAll{$t\geq 0$}
{
Compute $u_r(\cdot)$ in~\eqref{eq:mixed-init} to reach next $\pi^j\in \tau_r^t$\;
\If(\tcp*[f]{Model update}){$\mathcal{T}^t$ updated}
{Update product~$\mathcal{A}_p^t$ and plan~$\tau_r^t$ by~\eqref{eq:update-run}\;}
\If(\tcp*[f]{Temp. task}){$(\varphi_{\textup{temp}}^t,T_{sg})$ received}
{Update plan~$\tau_r^t$ by solving~\eqref{eq:delay}\;}
\If(\tcp*[f]{Human input}){$\|u_h(t)\|>0$}
{Compute control~$u(t)$ by~\eqref{eq:mixed-init}\;
Compute~$\zeta_h|_0^t$ by~\eqref{eq:human-traj}\;
Learn~$\beta_l^\star$ by Alg.~\ref{alg:learn-beta} and set $\beta_t^+=\beta_l^\star$\;
Update~$\tau_r^t$ by~\eqref{eq:update-run} given the learned $\beta_t^+$\;}
\Return $u(t)$, $\tau_r^t$, $\beta_t^+$
}
\end{algorithm}

\subsection{Computational Complexity}\label{subsec:complexity}
The process to synthesize~$\tau_r^t$ given $\mathcal{A}^t_p$ via Alg.~\ref{alg:complete} (in Line~1) and the plan revision given~$\mathcal{T}^t$ (in Line~5) both have complexity~$\mathcal{O}(|\mathcal{A}^t_p|^2)$~\cite{guo2015multi}. The adaptation algorithm for temporary tasks (in Line~7) has complexity~$\mathcal{O}(|R_p^t|^2)$. Lastly, the learning Alg.~\ref{alg:learn-beta} (in Line~11) has complexity~$\mathcal{O}(|R_h^\star|^2)$, where~$|R_h^\star|$ is the length of the optimal run from~\eqref{eq:minimal}.

\section{Case Study}\label{sec:case}
In this section, we present numerical studies both in simulation and experiment. The Robot Operation System (ROS) is used as the simulation and experiment platform. All algorithms are implemented in Python 2.7 and available online~\cite{mixed-package}. All computations are carried out on a laptop (3.06GHz Duo CPU and 8GB of RAM).

\subsection{Simulation}\label{subsec:simulate}

\begin{figure}[t]
\centering
\includegraphics[width =0.4\textwidth]{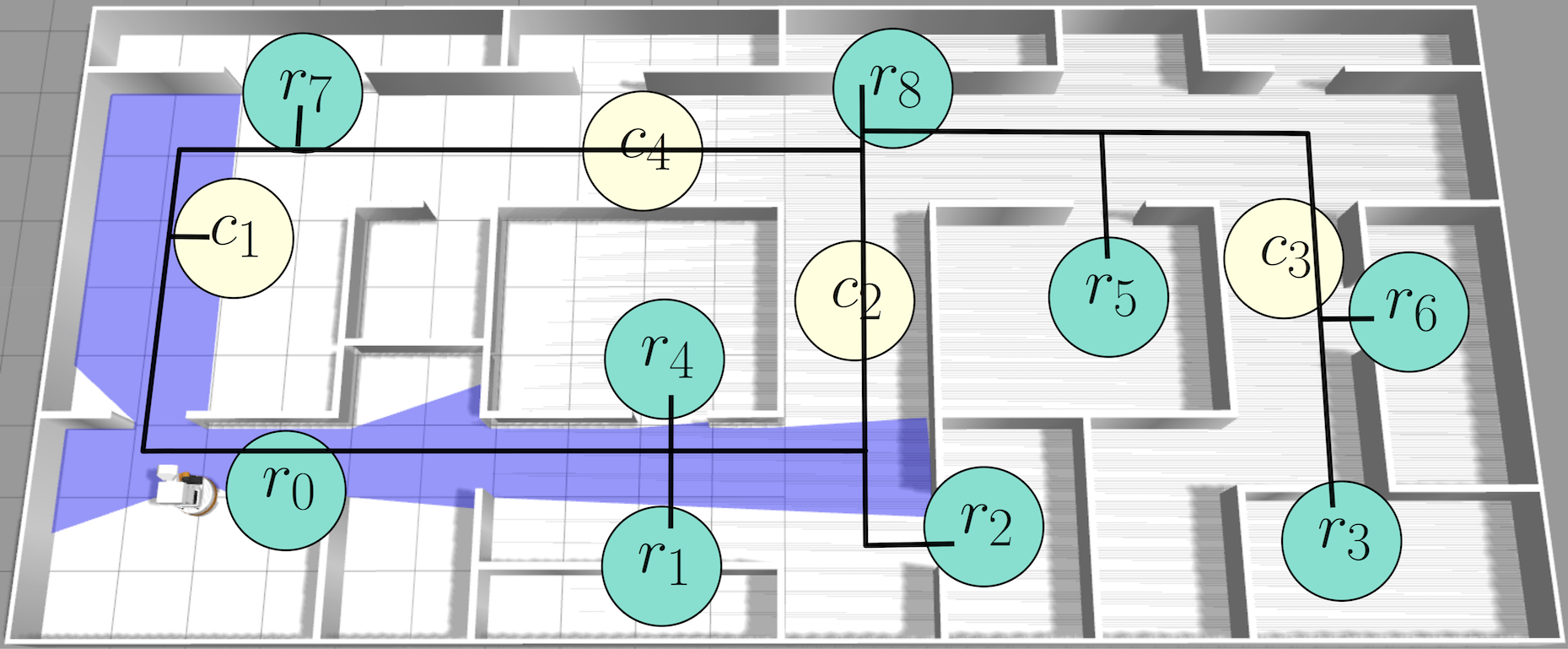}
\caption{Office environment in Gazebo with TIAGo robot, where the regions of interest and allowed transitions are marked.}
\label{fig:office-sim}
\end{figure}
\subsubsection{Workspace and Robot Description} Consider the simulated office environment in Gazebo as shown in Fig.~\ref{fig:office-sim} with dimension $100m\times 45m$, in which there are 9 regions of interest (denoted by~$r_0,\cdots,r_8$) and 4 corridors (denoted by $c_1$, $\cdots$, $c_4$). The transition relations  are determined by whether there exists a collision-free path from the center of one region to another, without crossing other regions.

We simulate the TIAGo robot from PAL robotics, of which the navigation control~$u_r(\cdot)$ with obstacle avoidance, localization and mapping are all based on the ROS navigation stack. The human operator monitors the robot motion through Rviz. Moreover, the control~$u_h(\cdot)$ from the operator can be generated from a keyboard or joystick, while the temporary task in LTL formulas~$\varphi_{\text{temp}}$ are specified via ROS messages. More details can be found in the software implementation~\cite{mixed-package} and simulation video~\cite{icra18-video}.

\begin{figure}[t]
\begin{minipage}[t]{0.495\linewidth}
\centering
   \includegraphics[width =1.02\textwidth]{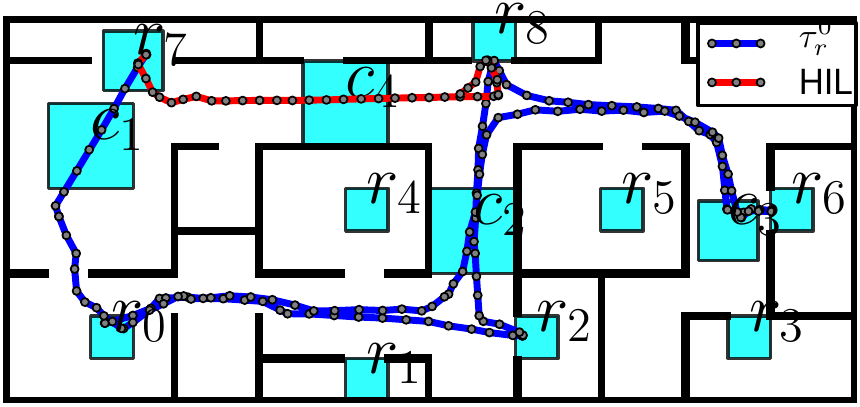}
  \end{minipage}
\begin{minipage}[t]{0.495\linewidth}
\centering
    \includegraphics[width =1.02\textwidth]{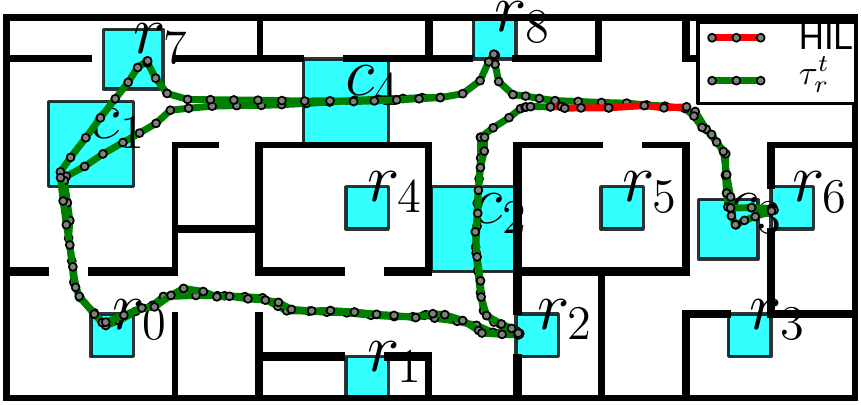}
  \end{minipage}
  \caption{The robot's trajectory in simulation case One, where the robot's initial plan is in blue, the human-guided part is in red, and the updated plan is in green.}
\label{fig:sim-traj-1}
\end{figure}

\subsubsection{Case One}
The hard task for \emph{delivery} is given by~$\varphi_{1,\text{hard}} = \big{(}\square \Diamond (r_0 \wedge \Diamond (r_7 \wedge \Diamond r_8))\big{)} \wedge \big{(}\square \Diamond (r_2 \wedge \Diamond (r_3 \vee r_6))\big{)} \wedge \big{(}\square \neg r_5\big{)}$, i.e., to transfer objects from $r_0$ to $r_7$ (then $r_8$) and from $r_2$ to $r_3$ (or $r_6$), while avoiding $r_5$ for all time. The soft task is $\varphi_{1,\text{soft}} =  (\square \neg c_4)$, i.e., to avoid $c_4$ if possible. It took $0.2s$ to compute the parameterized product automaton, which has $312$ states and $1716$ transitions. 
The parameter $\beta$ is initially set to a large value $30$, thus the initial plan satisfies both the soft and hard tasks but with a large cost due to the long traveling distance, as shown in Fig.~\ref{fig:sim-traj-1}.
During~$[700s,950s]$, the operator drives the robot to go through corridor~$c_4$ and reach $r_8$, which violates the soft task~$\varphi_{1,\text{soft}}$. As a result, $\beta$ is updated by Alg.~\ref{alg:learn-beta} and the final value is~$16.35$ after $20$ iterations with~$\varepsilon=0.2$, as shown in Fig.~\ref{fig:sim-beta-1}. Namely, the robot has learned that the  operator allows \emph{more violation} of the soft task to reduce the total cost.  The resulting updated plan is shown in Fig.~\ref{fig:sim-traj-1}.
Moreover, to demonstrate the ensured safety in Lemma~\ref{lemma:iss}, the human operator drives the robot towards~$r_5$ during~$[1250s,1350s]$, which is not allowed by $\varphi_{\textup{hard}}$.  The weighting function~$\kappa(\cdot)$ in the mixed controller~\eqref{eq:mixed-init} approaches $0$. Thus the robot still follows its updated plan and avoids~$r_5$.
The mixed control inputs during these periods are shown in Fig.~\ref{fig:sim-v-1-zoom}.

\begin{figure}[t]
\centering
\includegraphics[width =0.49\textwidth]{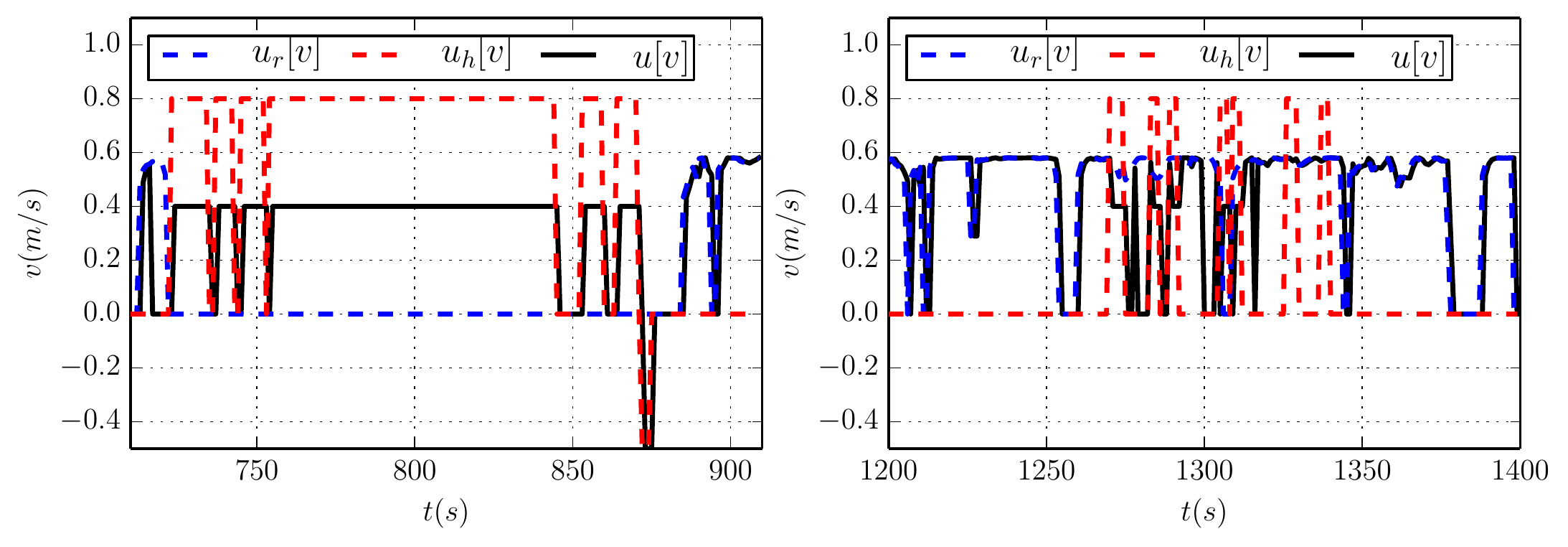}
\caption{The mixed control for linear velocity during two time periods when the human control (in red) is active.}
\label{fig:sim-v-1-zoom}
\end{figure}

\begin{figure}[t]
\begin{minipage}[t]{0.495\linewidth}
\centering
   \includegraphics[width =1.02\textwidth]{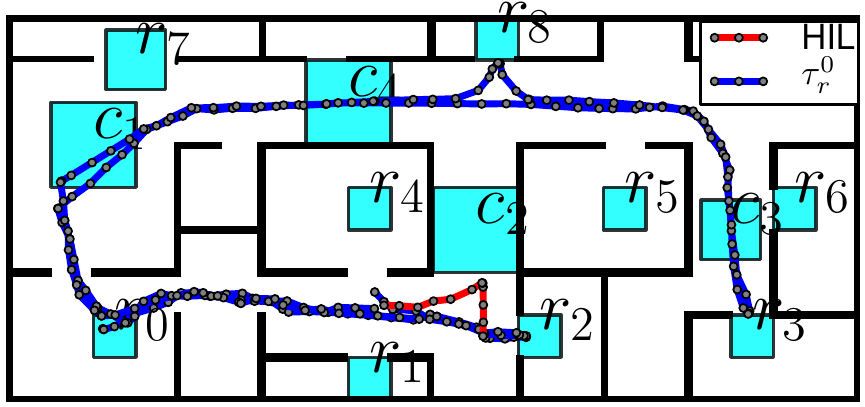}
  \end{minipage}
\begin{minipage}[t]{0.495\linewidth}
\centering
    \includegraphics[width =1.02\textwidth]{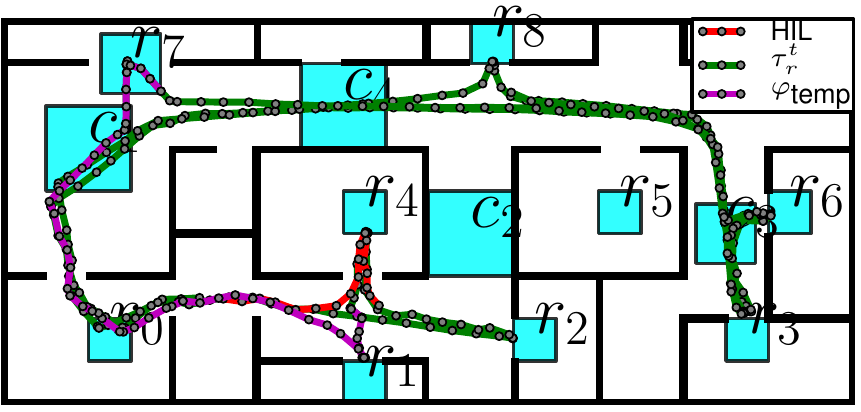}
  \end{minipage}
  \caption{The robot's trajectory in simulation case Two. The robot's trajectory while performing the temporary task is in magenta.}
\label{fig:sim-traj-2}
\end{figure}

\begin{figure}[t]
\begin{minipage}[t]{0.495\linewidth}
\centering
   \includegraphics[width =1.02\textwidth]{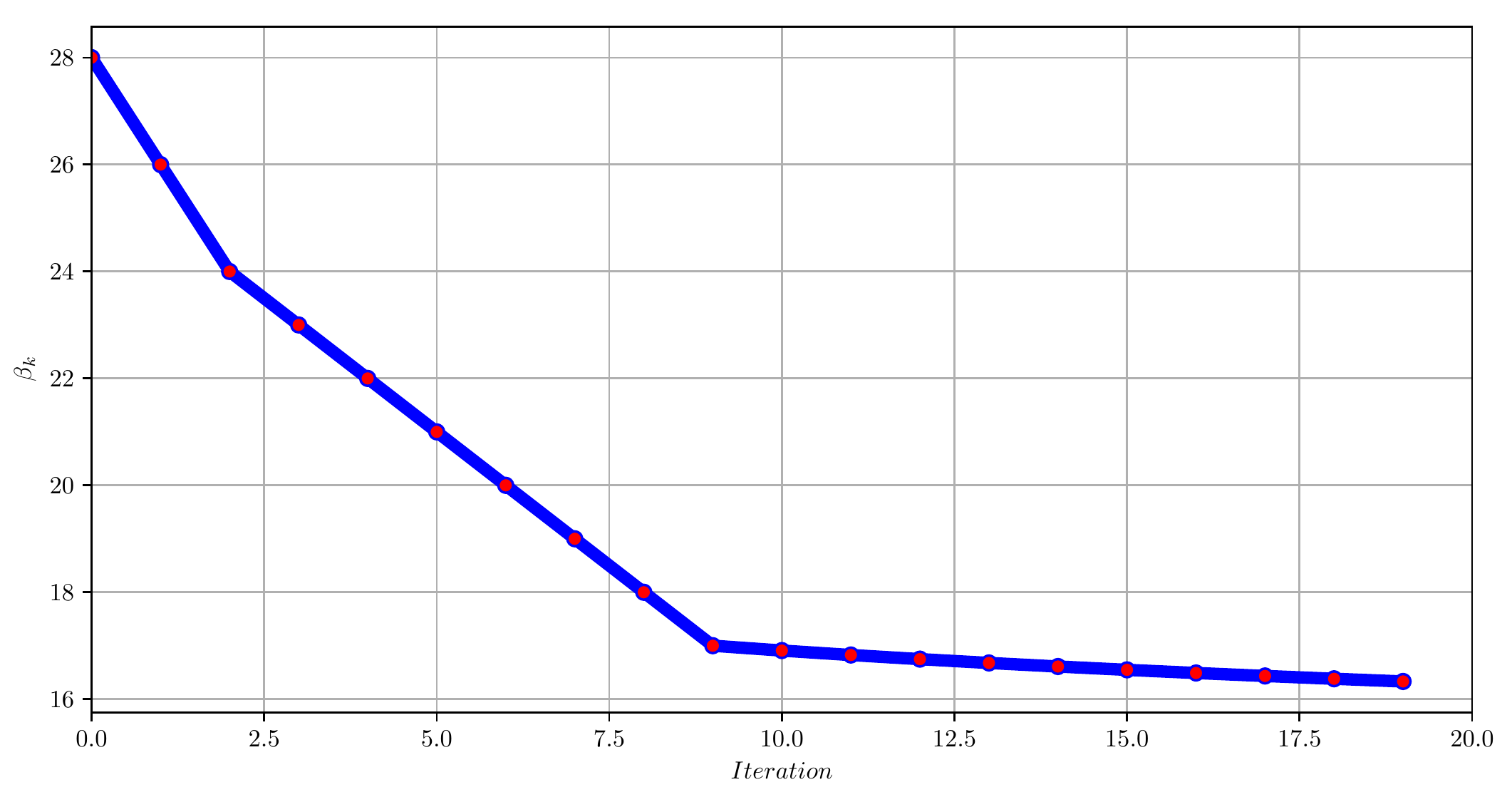}
  \end{minipage}
\begin{minipage}[t]{0.495\linewidth}
\centering
    \includegraphics[width =0.97\textwidth]{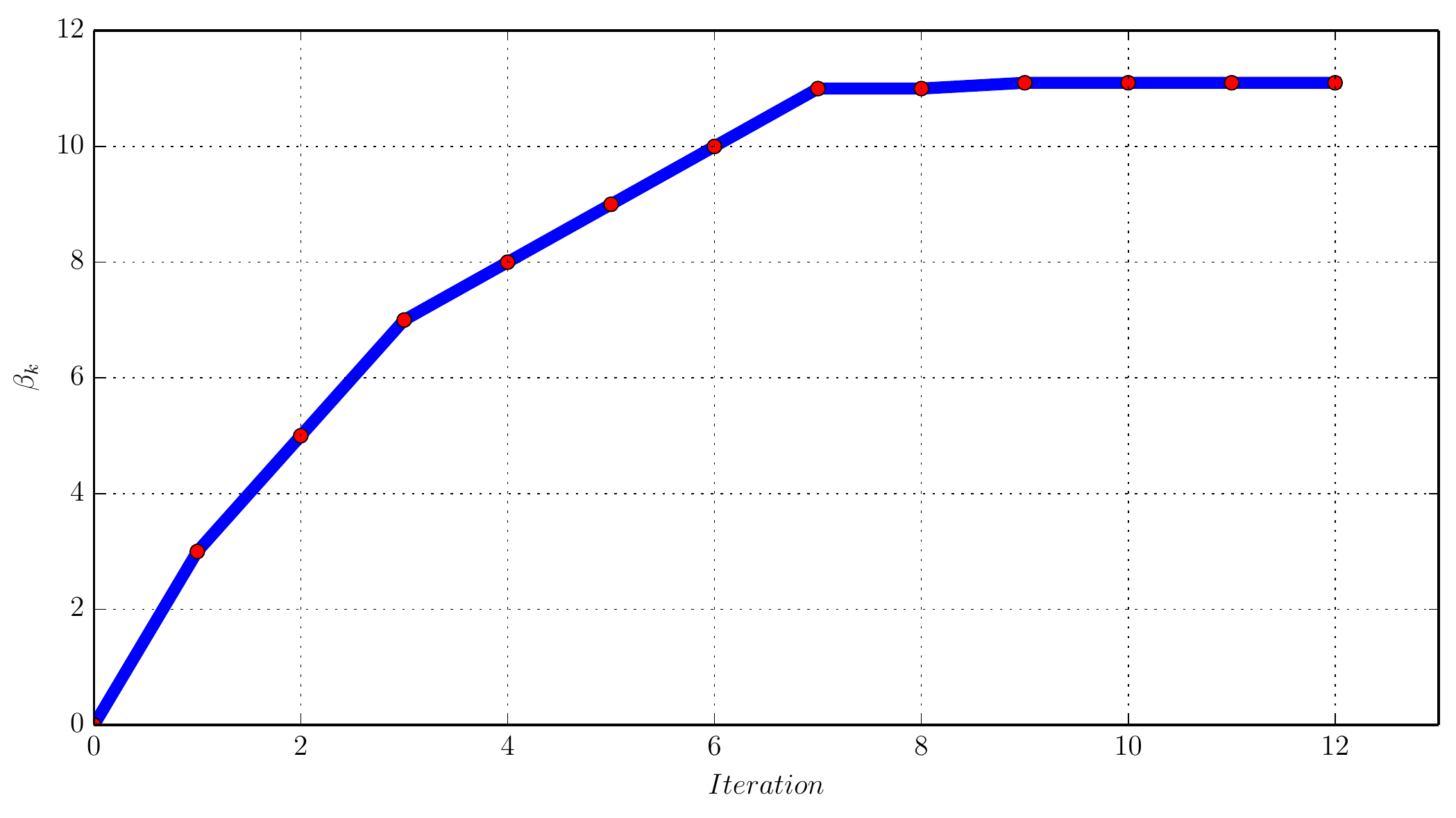}
  \end{minipage}
  \caption{Evolution of the learned value of $\beta$ in simulation case One (left) and case Two (right).}
\label{fig:sim-beta-1}
\end{figure}

\subsubsection{Case Two}
The hard task for \emph{surveillance} is given by~$\varphi_{2,\text{hard}} = (\square \Diamond r_2) \wedge (\square \Diamond r_3)\wedge (\square \Diamond r_8)$, i.e., to surveil regions $r_2$, $r_3$ and $r_8$ infinitely often.
The soft task for extra performance is $\varphi_{2,\text{soft}} = \square \Diamond \big{(}r_4 \rightarrow (\neg r_5 \mathsf{U} \Diamond r_6)\big{)}$, i.e., to collect goods from region $r_4$  and drop it at $r_6$ (without crossing $r_5$ before that).
Moreover, the workspace model in this case is \emph{different} from the initial model that the corridor $c_2$ has been blocked.
By following Alg.~\ref{alg:complete}, it took $0.17s$ to compute the product automaton, which has $418$ states and $3360$ transitions.  
Initially, $\beta=0$ meaning that the initial plan~$\tau_r^0$ only satisfies~$\varphi_{2,\text{hard}}$ while $\varphi_{2,\text{soft}}$ is fully relaxed.
During $[150s, 250s]$, the operator drives the robot to sense that the corridor $c_2$ has been blocked. As a result, the discrete plan~$\tau_r^0$ is updated such that the robot chooses to reach~$r_8$ from $r_2$ via $c_1$,  as shown in Fig.~\ref{fig:sim-traj-2}. 
Afterwards, during~$[1100s,1200s]$, the operator drives the robot to~$r_4$ after reaching $r_2$, which satisfies part of~$\varphi_{2,\text{soft}}$. As a result, $\beta$ is increased by Alg.~\ref{alg:learn-beta} to~$11.3$ after $12$ iterations with~$\varepsilon=0.1$, as shown in Fig.~\ref{fig:sim-beta-1}. Namely, the robot has learned that the soft task should be satisfied \emph{more}. Lastly, at time~$2100s$, the operator assigns a temporary task~$\varphi_{\text{temp}}=\Diamond(r_1 \wedge \Diamond r_7)$ with a deadline $2700s$, i.e., to deliver an object from $r_1$ to $r_7$. This temporary task is incorporated into $\tau_r^t$ and is fulfilled at $2400s$, which is shown in Fig.~\ref{fig:sim-traj-2}.

\subsection{Experiment}\label{subsec:experiment}
The experiment setup involves a TurtleBot within the office environment at Automatic Control Lab, KTH.  Details are omitted here due to limited space, which are given in the the software implementation~\cite{mixed-package} and experiment video~\cite{icra18-video}.

\begin{figure}[t]
\centering
\includegraphics[width =0.49\textwidth]{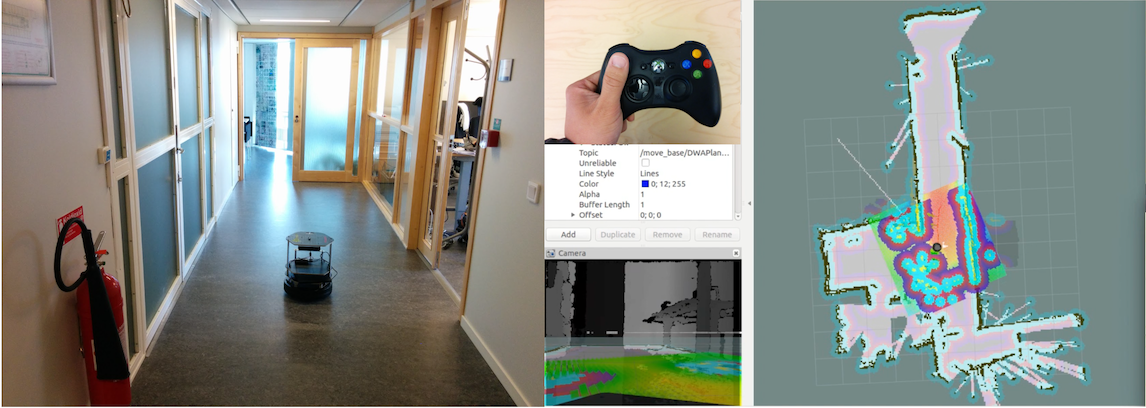}
\caption{The human-in-the-loop experiment setup, where the robot is controlled by both its autonomous controller and the human inputs.}
\label{fig:office-exp}
\end{figure}

\subsubsection{Workspace and Task Specification}
The office environment consists of three office rooms ($r_1$, $r_2$, $r_3$) and one corridor $r_0$, as shown in Fig.~\ref{fig:office-exp}. The robot's task specification is similar to case study Two above, i.e., the hard task is given by $\varphi_{\text{hard}}=\square\Diamond r_0 \wedge \square\Diamond r_1$ (to surveil regions $r_0$ and $r_1$) while the soft task is  $\varphi_{\text{soft}}=\square\Diamond r_2 \wedge \square\Diamond r_3$ (to surveil regions $r_2$ and $r_3$).
The TurtleBot is controlled via ROS navigation stack and behaves similarly to the TIAGo robot in Section~\ref{subsec:simulate}.

\subsubsection{Experiment Results}
Since~$\beta$ is initially set to $0$, the robot only surveils $r_0$ and $r_1$ for the hard task, as shown in Fig.~\ref{fig:exp-traj}. From~$t=59s$, the operator starts driving the robot towards $r_2$ and back to $r_0$ until $t=137s$. As a result, the estimated~$\beta_t$ is updated by Alg.~\ref{alg:learn-beta} given the robot's past trajectory. The final convergence value is~$1.58$ with $\varepsilon=0.01$ after 15 iterations. Then updated plan is shown in Fig.~\ref{fig:exp-traj} which intersects with not only regions~$r_0$ and $r_1$ for the hard task, but also regions~$r_2$ and $r_3$ for the soft task. Notice that the operator only needs to interfere the robot's  motion for a small fraction of the operation time. 

\begin{figure}[t]
\centering
\includegraphics[width =0.49\textwidth]{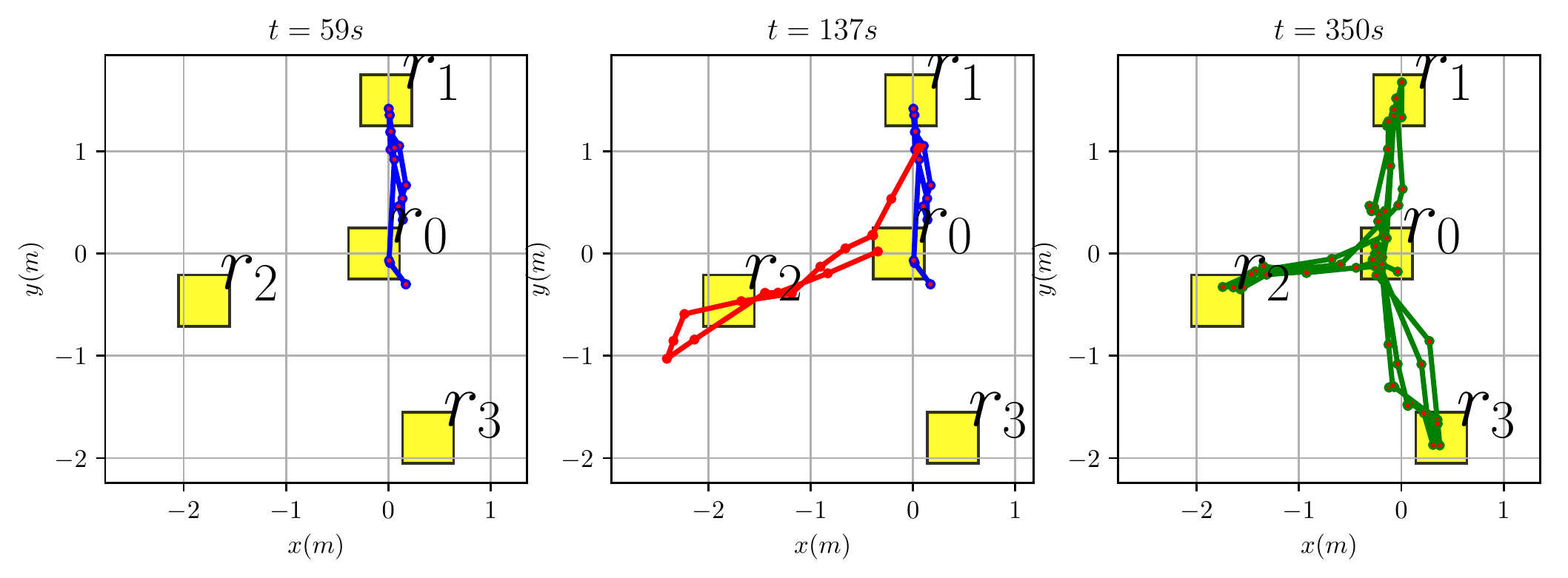}
\caption{The robot's trajectory in the experiment study, where the robot's initial plan is in blue (left), the human-guided segment is in red (middle), and the updated plan is in green (right). 
}
\label{fig:exp-traj}
\end{figure}

\section{Summary and Future Work}\label{sec:future}
In this paper, we present a human-in-the-loop task and motion planning strategy for mobile robots with mixed-initiative control. The proposed coordination scheme ensures the satisfaction of high-level LTL tasks given the human initiative both through continuous control inputs and discrete task assignments. Future work includes consideration of multi-robot systems. 
\bibliography{meng.bib}

\end{document}